\documentclass{article}

 \usepackage[preprint]{neurips_2024_calm}

\usepackage[utf8]{inputenc} 
\usepackage[T1]{fontenc}    
\usepackage{hyperref}       
\usepackage{url}            
\usepackage{booktabs}       
\usepackage{amsfonts}       
\usepackage{nicefrac}       
\usepackage{microtype}      
\usepackage{xcolor}         
\usepackage{soul}
\usepackage{algorithm}
\usepackage{algpseudocode}
\usepackage{enumitem}
\usepackage{amsmath}
\usepackage{amssymb}
\usepackage{amsthm}
\usepackage{graphicx}
\usepackage{booktabs}
\newtheorem*{theorem}{Theorem}

\newcommand{\skiptext}[1]{}

\title{LLM-initialized Differentiable Causal Discovery}

\author{%
  Shiv Kampani \\
  SandboxAQ \\
  \texttt{shiv.kampani@sandboxaq.com}
  \And
  David Hidary \\
  SandboxAQ \\
  \texttt{david.hidary@sandboxaq.com}
  \And
  Constantijn van der Poel \\
  SandboxAQ \\
  \texttt{constantijn@sandboxaq.com}
  \And
  Martin Ganahl \\
  SandboxAQ \\
  \texttt{martin.ganahl@sandboxaq.com}
  \And
  Brenda Miao \\
  SandboxAQ \\
  \texttt{brenda.miao@sandboxaq.com}
}

\begin{document}

\maketitle

\begin{abstract}
The discovery of causal relationships between random variables is an important yet challenging problem that has applications across many scientific domains. Differentiable causal discovery (DCD) methods are effective in uncovering causal relationships from observational data; however, these approaches often suffer from limited interpretability and face challenges in incorporating domain-specific prior knowledge. In contrast, Large Language Models (LLMs)-based causal discovery approaches have recently been shown capable of providing useful priors for causal discovery but struggle with formal causal reasoning. In this paper, we propose LLM-DCD, which uses an LLM to initialize the optimization of the maximum likelihood objective function of DCD approaches, thereby incorporating strong priors into the discovery method. To achieve this initialization, we design our objective function to depend on an explicitly defined adjacency matrix of the causal graph as its only variational parameter. Directly optimizing the explicitly defined adjacency matrix provides a more interpretable approach to causal discovery. Additionally, we demonstrate higher accuracy on key benchmarking datasets of our approach compared to state-of-the-art alternatives, and provide empirical evidence that the quality of the initialization directly impacts the quality of the final output of our DCD approach. LLM-DCD opens up new opportunities for traditional causal discovery methods like DCD to benefit from future improvements in the causal reasoning capabilities of LLMs.
\end{abstract}

\section{Introduction}

Discovering causal relationships is a fundamental task across scientific fields including epidemiology, genetics, and economics. For a variety of reasons – ethical, logistical, legal – it may not be possible to conduct controlled experiments or interventional studies to generate Causal Graphical Models (CGM) that allow for causal inference. Consequently, there has been a shift towards the development of \textit{causal discovery} methods that infer CGMs from observations about the variables alone.

\subsection{Background}

Causal discovery is the problem of learning a CGM from a set of observed data points. 
Each observation $n, n = 1,\dots,N$, can be characterized as a specific realization of $d$ random variables $V = \{v_1, \dots, v_d\}$, where the variable $v_j$ can take on discrete or continuous values. In this work we restrict ourselves to discrete random variables. We use $\mathbf{x}^n$ to denote the vector of $d$ discrete values $x^n_j, j = 1,\dots, d$ of the $n$-th observation,  and denote the table of all $N$ observations as $X \equiv \{ \mathbf{x}^n \}^N_{n = 1}$. 

A CGM for the variables $V$ consists of two components: 1) a directed acyclic graph (DAG) $G = (V, E)$ with  nodes $V$ and {\it directed} edges $E\equiv\{e_1,\dots e_M\} \subseteq V \times V$  of (ordered) pairs $e \equiv (v_a, v_b)$, with the order implying causation, i.e. $v_a$ causes $v_b$ in this notation. 2) a set of $d$ conditional probability distributions $p ( \, v_j \; | \; \mathrm{Pa}(v_j, G) )$, where $\forall j \in [\, d \,]$, $\mathrm{Pa}(v_j, G) \subseteq V $is the set of \textit{causal parents} or direct causes of $v_j$ according to $G$. In this work we do not consider interventions, although they can be incorporated into our framework.

Since causal discovery is an NP-hard problem  (\cite{chickering}), a variety of approaches have been studied to improve the efficiency and accuracy of causal discovery. 
These methods often fall into three categories - score-based methods (SBMs), differentiable causal discovery (DCD), and Large Language Model (LLM)-based approaches.

\textbf{Score-based methods.} SBMs formulate causal discovery as the maximization of a log-likelihood objective function 
\begin{equation}\label{eq:sbm_loss}
    \mathcal{L}(G, \theta ; X) = \frac{1}{N}\sum^N_{n = 1} \sum^d_{j = 1} \log p(v_j =x^n_j \, ; \, X, \theta, G) \;
\end{equation}
with respect to parameters $\theta$ and the graph $G$  (regularization penalties, typically added to the cost function, are omitted here for simplicity). The functions $p(v^n_j \, ; \, X, \theta, G)$ are ansatz functions for the conditional probability distributions $p ( \, v_j \; | \; \mathrm{Pa}(v_j, G) )$. \cite{brouillard2020differentiablecausaldiscoveryinterventional} showed that, under certain regularity assumptions (see Appendix \ref{app:proof}), the maximizer of this objective function is {\it Markov equivalent} to the ground-truth CGM.

In this approach, acyclicity and directedness can be enforced by applying suitable constraints to the discrete optimization problem.  An advantage is that the optimization is amenable to standard combinatiorial optization methods, however,  due to the super-exponential growth of the solution space with the size of the graph, score-based approaches become quickly intractable (\cite{Meek1997}).

\textbf{Differentiable causal discovery.} More recent approaches formulate causal discovery as a continuous optimization problem over the adjacency matrix $A_{\theta}$ of the graph $G$ and parameters $\theta$. The notation $A_{\theta}$  here means that the adjacency matrix can be an arbitrary differentiable function of the parameter $\theta$. Acyclicity and directedness constraints of the graph are incorporated into the optimization by a differentiable penalty function  $h(A_{\theta})$, i.e.
\begin{equation}\label{eq:dcd_loss}
    \mathcal{L}(A_{\theta}, \theta ; X) = \frac{1}{N}\sum^N_{n = 1} \sum^d_{j = 1} \log p(v_j =x^n_j \, ; \, X, \theta, A_{\theta})  - \beta h(A_{\theta})\;
\end{equation}

The function $h(A)$ needs to satisfy the conditions  (\cite{nazaret2024stabledifferentiablecausaldiscovery})
\begin{align}
 h(A) \geq 0&\\
 h(A) = 0& \iff G \;\textrm{a directed acyclic graph}.
\end{align}
The objective is then maximized using standard gradient ascent optimization for parameters $A_{\theta}$ and $\theta$. Acyclicity and directedness are in this case not strictly enforced during the optimization, but approached during the maximisation of the objective function. 

\cite{zheng2018dagstearscontinuousoptimization} were the first to propose NOTEARS, a differentiable causal discovery approach based on a trace-exponential acyclicity constraint $h(A_{\theta})=\textrm{Tr}(\exp(A_{\theta}))-d$.  \cite{bello2023dagmalearningdagsmmatrices} later on introduced DAGMA, a direct improvement over NOTEARS owing to an alternate log-det-based acyclicity constraint. 
\cite{nazaret2024stabledifferentiablecausaldiscovery}  proposed a more stable acyclicity constraint in the form of the spectral radius of $A_\theta$, i.e. $h(A_{\theta}) = |\lambda_d|$, with $\lambda_d$ the largest eigenvalue of the matrix. Note that the spectral radius of the adjacency matrix of a DAG is identically 0 ($A$ is a nilpotent matrix in this case). We employ the spectral acyclicity constraint in this work.

\textbf{LLM-based approaches.} LLMs have shown promise in being able to evaluate pairwise causal relationships between variables of interest (\cite{kıcıman2023causalreasoninglargelanguage, liu2024largelanguagemodelscausal}). However, LLM-based approaches are often unable to perform formal causal reasoning and can exhibit inconsistences; it is also difficult to distinguish true causal reasoning from memorization (\cite{jin2024cladderassessingcausalreasoning}). Also, unlike previous approaches, LLMs are not designed to leverage observational data effectively. 


\subsection{Problem setup.} Recent approaches have sought to merge the advantages of LLM-based and SBM approaches. \cite{darvariu2024largelanguagemodelseffective} showed that LLMs can provide effective priors to improve score-based algorithms. LLMs have also been used specify constraints for SBMs (\cite{ban2023querytoolscausalarchitects}), by orienting edges in partial CGMs discovered by other numerical approaches (\cite{long2023causaldiscoverylanguagemodels}), or by providing a ``warm-start" or initial point for combinatorial search-based methods (\cite{vashishtha2023causalinferenceusingllmguided}). These results suggest \textbf{LLMs may also be able to complement and improve the performance of state-of-the-art DCD methods}. However, integrating LLM-based causal discovery methods with DCD is challenging due to \textit{non-interpretable} adjacency matrices.


The function $p( \cdot )$ used to model conditional distributions in the DCD loss, Eq. (\ref{eq:dcd_loss}), is represented using a neural network with parameters $\theta$. The adjacency matrix $W_\theta$ of the CGM DAG is then implicitly defined from $\theta$, by taking a norm over the first layer of the network. \cite{Waxman_2024} show that when $p( \cdot )$ is modelled using a Multi-Layer Perceptron (MLP) with sigmoid or ReLU activation (as in NOTEARS, DAGMA, and SDCD), the implicit adjacency matrix, $W_\theta$, from the first layer of the network may be arbitrarily different from the true causal relationships found by taking derivatives over all model layers.
The MLP parameters $\theta$ do not directly relate to the implied adjacency matrix $W_\theta$, so the implicit adjacency matrix $W_\theta$ is \textbf{non-interpretable}.

Said another way, suppose an LLM (correctly) reasons that variable $v_a$ directly causes $v_b$. It is clear that $(W_\theta)_{ab} = 1$. However, it is not clear how to modify the MLP parameters, $\theta$, to encode this information and leverage information from LLMs: hence the non-interpretability.

\subsection{Contributions}
In this work we propose a novel combination of large language models (LLM) with DCD. The two key aspects of our approach are:
\begin{enumerate}
    \item an ansatz function $p(v_j =x^n_j \, ; \, X, A)$ in Eq.(\ref{eq:mleinterp}) [\texttt{MLE-INTERP}], which depends on elements $a_{jk}\geq 0$ of an explicitly defined adjacency matrix $A$ as the only variational parameters. \texttt{MLE-INTERP} does not rely on an MLP or neural network to model conditional distributions.
    \item the usage of LLMs for parameter initialization of the adjacency matrix $A$. This is possible because we use an ansatz function depending on an explicitly defined adjacency matrix. 

\end{enumerate}

We make the following assumptions in our approach:
\begin{enumerate}[noitemsep]
    \item \textit{Faithfulness, causal Markov condition:} conditional independence holds in the observational data table $X$ if and only if the corresponding d-separation holds in the CGM.
    \item \textit{Causal sufficiency:} there are no hidden or latent confounders.

\end{enumerate}

We coin our approach LLM-initialized Differentiable Causal Discovery (LLM-DCD). Our method is applicable to datasets with discrete-valued variables, and we do not make any assumptions about conditional distributions. This addresses a limitation in DCD methods like SDCD that assume that conditionals are normally-distributed. To our knowledge, LLM-DCD is the first method to integrate LLMs with differentiable causal discovery.

We prove that our model satisfies the standard regularity assumptions for DCD (\cite{brouillard2020differentiablecausaldiscoveryinterventional}, Appendix \ref{app:proof}) and benchmark the performance of LLM-DCD on five datasets ranging from 5-70 discrete-valued variables and limited observational data. LLM-DCD outperforms existing methods and scales reasonably with time. We also empirically show that the quality of LLM-DCD depends on the quality of its initialization; thus, the development of higher quality LLM-based causal discovery methods is expected to have direct impact on the quality of LLM-DCD.

All code from this work is available at: \url{https://github.com/sandbox-quantum/llm-dcd}

\section{LLM-initialized Differentiable Causal Discovery}

We consider the differentiable objective function
\begin{equation}\label{eq:ourloss}
    \mathcal{L}(A ; X) = \frac{1}{n}\sum^N_{n = 1} \sum^d_{j = 1} \log \texttt{MLE-INTERP}(x^n_j ; X, A) - \alpha ||A||_1 - \beta_t |\lambda_d |\;.
\end{equation}
with an ansatz function $\texttt{MLE-INTERP}(x^n_i; X,A)$ to model the conditional probability $ p(v_i = x_i|\{v_{k} = x_{k}\}_{k\neq i})\equiv\frac{ p(x_i,\{x_k\}_{k\neq i})}{p(\{x_k\}_{k\neq i})}$ of observation $x_i$ given values $\{x_k\}_{k\neq i}$ for the variables $\{v_k\}_{k\neq i}$ from the set of training data $X$. 

Specifically, we use a maximum-likelihood based estimator of this conditional probability which is computed from ratios of frequency counts of observations in the training data. In the following we use $\textrm{cnt}(x_{j_1}, x_{j_2},\dots x_{j_{M}})$ to denote the number of samples in $X$ with values $x_{j_1}, x_{j_2},\dots x_{j_{M}}$.

Consider as an example the case of two variables $v_1, v_2$, with $a_{21}$ the only possible non-zero element of $A$. We use the following ansatz for the conditional distribution $p(v_1=x_1|v_2=x_2)$ is in this case:
\begin{equation}
     \texttt{MLE-INTERP}(x_1 ; X, A) =  \frac{\textrm{cnt}(x_1)(1-a_{21})+\textrm{cnt}(x_1, x_2)a_{21}}{N(1-a_{21} ) + \textrm{cnt}(x_2)a_{21}}.
\end{equation}

In the limits of $a_{21} = 1$ and $a_{21} = 0$, the function reduces to the expected expressions $   \frac{\textrm{cnt}(x_1, x_2)}{\textrm{cnt}(x_2)}$ and $ \frac{\textrm{cnt}(x_1)}{N}$, respectively. For values $0< a_{21} < 1$ the nominator and denominator are taken to be a linear interpolation between the counts of the two edge cases. This ansatz can be straightforwardly generalized to $d$ variables:
\small
\begin{equation}
\label{eq:mleinterp-1}
 \texttt{MLE-INTERP}(x_i ; X, A) = \frac{\textrm{cnt}(x_i)\prod_{k\neq i}^d(1-a_{ki}) +\sum_{j\neq i} \textrm{cnt}(x_i, x_j)a_{ji}\prod_{k\neq \{i, j\}}(1-a_{ki}) + \hdots}{N\prod_{k\neq i}(1-a_{ki}) +\sum_{j} \textrm{cnt}(x_j)a_{ji}\prod_{k\neq j}(1-a_{ki}) + \hdots}.
\end{equation}
\normalsize

We emphasize that despite the seeming complexity, this function can be evaluated in $\mathcal{O}(Nd)$ time;  see also Algorithm \ref{alg:mleinterp} in the Appendix.
Eq.(\ref{eq:mleinterp-1}) can be rewritten in a more succinct form: using the notation  $\bar x_j$ to denote all values different from $x_j$, such that e.g. $\textrm{cnt}(x_{1}, \bar x_{2})$ are the number of samples in $X$ with values $v_1 = x_1$ and values $v_2 \neq x_2$, and with the shorthand 
\begin{align}
   \delta_k(x_j) = 
    \begin{cases}
        x_j\; {\rm if}\; k=0\\
        \bar x_j\; {\rm if}\; k=1
    \end{cases}
\end{align}
Eq.(\ref{eq:mleinterp-1}) can be recast into
\begin{align}
    \texttt{MLE-INTERP}(x_i ; X, A) = \frac{\sum_{\{i_k\}_{k = 1\dots d-1} \in \{0,1\}} \textrm{cnt}(x_i,\{\delta_{i_k}(x_{j_k})\}_{j_k\neq i})\prod_{k=1}^{d-1} (1-i_k a_{j_ki})}{\sum_{\{i_k\}_{k = 1\dots d-1} \in \{0,1\}}\textrm{cnt}(\{\delta_{i_k}(x_{j_k})\}_{j_k\neq i})\prod_{k=1}^{M_i} (1-i_ka_{j_ki})}.\label{eq:mleinterp}
\end{align}
where $\{x_k\}_{k\neq i}$ denotes all variables $x_k$  with $k\neq i$. 

Eq.(\ref{eq:mleinterp})  evidently is a rational function of the the elements of the adjacency matrix $A$, for which gradients w.r.t. $a_{ji}$ can be computed efficiently.  To improve stability during the optimization, we replace factors $(1-i_ka_{j_ki})$ in Eq.(\ref{eq:mleinterp}) with a third order polynomial  $g(1-i_ka_{j_ki})$. We chose $g$ as follows because it is a smooth function that satisfies $g(0) = 0, g(1) = 1$, and $\forall x \in [0, 1], g'(x) \ne 0$.
\begin{equation}
    g(x) = 0.15x + 2.55x^2 - 1.7x^3.
\end{equation}


The objective function Eq.(\ref{eq:ourloss}) is maximized using an Adam optimizer (\cite{adam}) with mini-batch gradient ascent.
Note that spectral acyclicity, $h(A) = 0$, is included in the objective via the penalty $\beta_t \cdot h(A)$ with a step-dependent coefficient  $\beta_t$. The optimization is carried out over two stages, as in described in \cite{nazaret2024stabledifferentiablecausaldiscovery}. In the first stage, $\beta_t$ is kept 0 for $t_1$ iterations, with the goal of maximizing log-likelihood. In the second stage, $\beta_t$ is incremented by $\delta$ at every timestep, with the goal of ensuring acyclicity, for at least $t_2$ iterations, until $A$ converges.

\subsection{LLM-initialization}
Since the model parameters are exactly the adjacency matrix $A$ of the CGM, we may choose to initialize the optimization process with a ``warm start" adjacency matrix $A_0$ provided by an LLM.
We consider ``warm starts" provided both by the pairwise (PAIR) and breadth-first-search (BFS):
\begin{enumerate}[noitemsep]
    \item \textbf{Pairwise LLM queries.} This is a straightforward LLM-based solution to the problem of causal discovery introduced by \cite{kıcıman2023causalreasoninglargelanguage}. For every pair of variables $(v_a, v_b)$, an LLM is asked to reason whether $v_a$ directly causes $v_b$, $v_b$ directly causes $v_a$, or that there is no direct causal relationship between the two variables.
    \item \textbf{LLM breadth-first-search (BFS).} The number of queries made by the pairwise LLM method scales as $O(d^2)$, where $d$ is the number of variables. \cite{jiralerspong2024efficientcausalgraphdiscovery} proposed a three-stage BFS-based algorithm that requires only $O(d)$ queries.
\end{enumerate}
For both methods, we also provided the LLMs with pairwise correlational coefficients where relevant.
In Section \ref{sec:results}, we provide empirical evidence that LLM-initialization improves the quality of LLM-DCD over previous state-of-the-art methods.

\section{Experimental Setup}\label{sec:exp}
We benchmark the performance of LLM-DCD against previous score-based method GES, differentiable methods SDCD and DAGMA, and the aforementioned LLM-based methods (PAIR and BFS) on the following CGM datasets from the \texttt{bnlearn} package (\cite{bnlearnpaper}): \texttt{cancer} (5 variables, 4 causal edges), \texttt{sachs} (11 variables, 17 causal edges), \texttt{child} (20 variables, 25 causal edges), \texttt{alarm} (37 variables, 46 causal edges), and \texttt{hepar2} (70 variables, 123 causal edges).

All observational data ($n = 1000$ observations for each experiment) are generated by independent forward-sampling from the ground-truth CGM joint distribution. We run  LLM-DCD with a random initialization, initialized with the output of PAIR (denoted LLM-DCD-P or LLM-DCD (PAIR)), and initialized with the output of BFS (denoted LLM-DCD-B or LLM-DCD (BFS)). Each method was run with three global seeds (0,1,2) on each dataset. Consistent with previous works, we run LLM-DCD with a fixed set of hyperparameters (Appendix \ref{app:hyperp}) in all experiments. 

\textbf{Metrics} We report performance of LLM-DCD and baseline models using structural Hamming distance (SHD) between the predicted and true CGMs. SHD is a standard causal discovery metric and is measured as the number of causal edge insertions, deletions, or flips required to transform the current CGM DAG into the ground-truth CGM DAG. We also provide results based on precision, recall, F1-score, and runtime (seconds). Results are reported as means $\pm$ standard deviation.

\section{Results}  \label{sec:results}

\textbf{Theoretical results.} In Appendix \ref{app:proof}, we prove that the \texttt{MLE-INTERP} model used by LLM-DCD satisfies regularity assumptions introduced by \cite{brouillard2020differentiablecausaldiscoveryinterventional}. As a corollary, in limit of infinite samples ($n \to \infty$), LLM-DCD outputs a CGM that is \textit{Markov equivalent} to the ground-truth CGM. We also provide an analysis of the time-complexity of LLM-DCD in Appendix \ref{app:time}. The current implementation of LLM-DCD calculates $h'(A)$ naïvely in $O(d^3)$-time, but future implementations are likely to benefit from utilizing the recently developed $O(d^2)$-time power-iteration algorithm from \cite{nazaret2024stabledifferentiablecausaldiscovery}.

\begin{table}[!ht]
\centering
\caption{Performance (SHD) of causal discovery methods on benchmarking datasets}
\label{SHD_table}
\begin{tabular}{lccccc}
\toprule
\multicolumn{1}{c}{} & \multicolumn{5}{c}{\textbf{Dataset}} \\ 
\cmidrule(lr){2-6}
 & Cancer & Sachs & Child & Alarm & Hepar2 \\ 
\textbf{Method} & (n=5) & (n=11) & (n=20) & (n=37) & (n=70) \\ 
\midrule
SDCD        & $2.00 \pm 0.00$  & $\mathbf{12.3 \pm 2.08}$ & $35.7 \pm 0.58$  & $131. \pm 11.6$ & $230. \pm 13.6$  \\
GES         & $2.67 \pm 1.15$  & $\mathbf{14.3 \pm 1.53}$& $\mathbf{26.7 \pm 3.51}$& $60.5 \pm 0.71$ & $-$   \\
DAGMA       & $4.00 \pm 0.00$           & $25.3 \pm 1.53$          & $32.7 \pm 0.58$ & $88.3 \pm 9.29$ & $272. \pm 6.43$   \\
PAIR        & $4.00 \pm 0.00$           & $31.0 \pm 0.00$          & $97.3 \pm 10.2$ & $207. \pm 9.54$ & $-$   \\
BFS         & $\mathbf{1.00 \pm 1.00}$  & $16.3 \pm 2.08$          & $36.3 \pm 3.51$ & $64.0 \pm 5.29$ & $173. \pm 13.1$   \\
LLM-DCD     & $3.67 \pm 0.58$           & $23.7 \pm 9.02$          & $49.7 \pm 5.51$ & $58.0 \pm 1.73$ & $318. \pm 8.89$   \\
LLM-DCD-P   & $2.67 \pm 0.58$           & $21.0 \pm 3.46$          & $41.0 \pm 3.61$ & $60.3 \pm 18.8$ & $-$   \\
LLM-DCD-B   & $\mathbf{1.00 \pm 1.00}$  & $\mathbf{14.3 \pm 0.58}$& $\mathbf{23.7 \pm 0.58}$ & $\mathbf{36.3 \pm 5.51}$  & $\mathbf{140. \pm 6.66}$ \\

\bottomrule
\end{tabular}
\end{table}

\textbf{Evaluation Results}
Table \ref{SHD_table} shows the performance of LLM-DCD-B on several benchmarking datasets of varying sizes. Lower SHD indicates better performance. LLM-DCD (BFS) outperformed all baseline SBM, DCD, and LLM based approaches in the Alarm and Hepar2 datasets, and achieved results that were comparable to the top-performing models on the Cancer, Sachs, and Child datasets. Results are not reported for methods with intractable runtimes, including GES, PAIR, and LLM-DCD (Pair) on the Hepar2 dataset.

In Figure \ref{fig:other_metrics_plot}, we observe similar trends for F1-score, Precision, and Recall, with LLM-DCD-B consistently outperforming other methods across datasets. The runtime of LLM-DCD scales worse than that of SDCD, which remains roughly constant, although LLM-DCD (BFS) was still more efficient than GES, PAIR, and DAGMA (Figure \ref{fig:other_metrics_plot}).

We further show that initialization of the adjacency matrix in LLM-DCD affects performance, with higher quality initializations tending to result to better performance across metrics and datasets. LLM-DCD (BFS) tended to outperform LLM-DCD (PAIR) and the randomly initialized LLM-DCD. Additionally, we show that the randomly initialized LLM-DCD method has performance that is comparable to existing DCD methods like DAGMA and SDCD.

\begin{figure}[!h]
\label{fig:other_metrics_plot}
  \centering
  \includegraphics[width=0.8\linewidth]{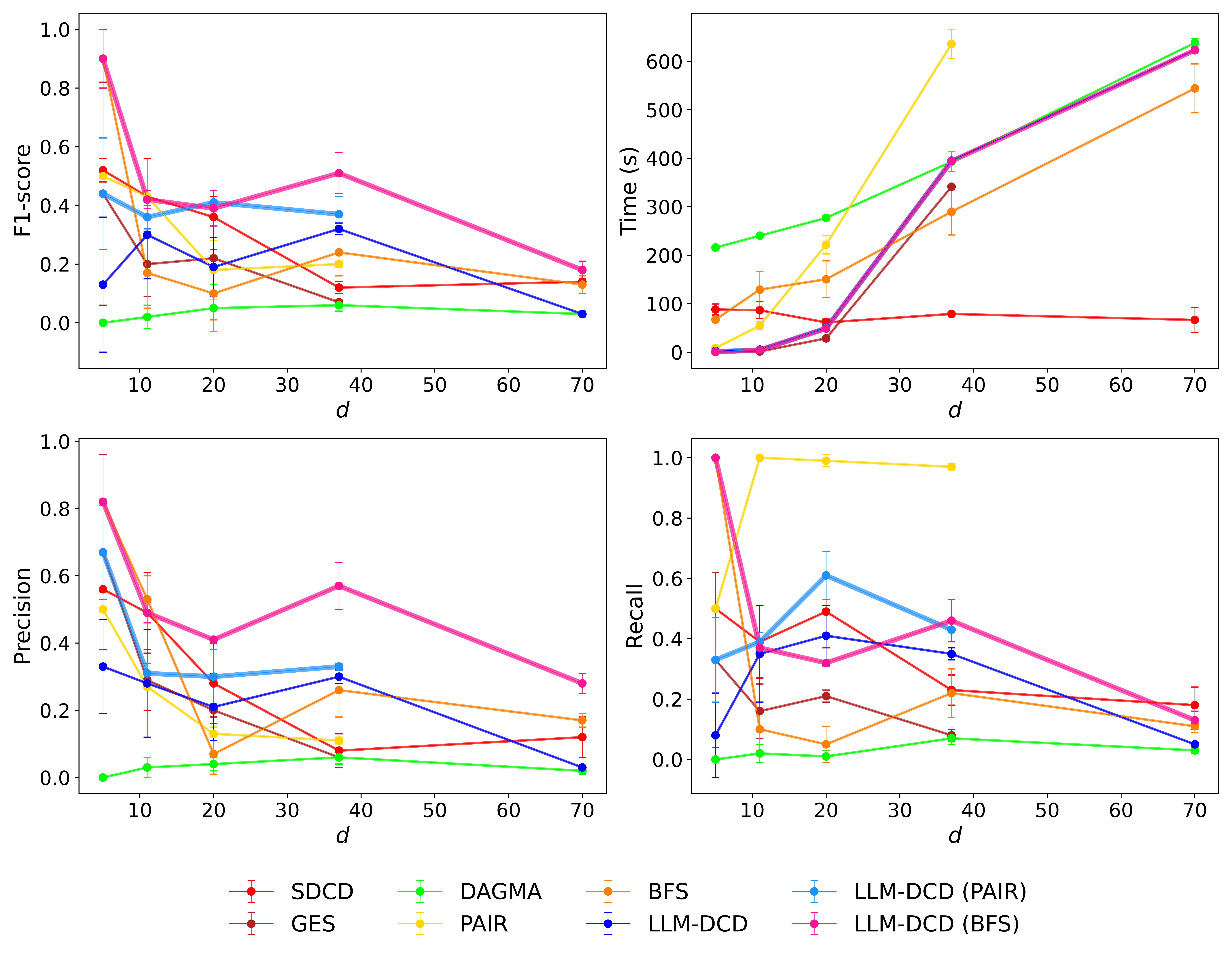}
  \caption{F1-score, precision, recall, and runtime of causal discovery methods on observational datasets of different sizes ($d$).}
\end{figure}

\section{Conclusion}
We developed LLM-DCD, integrating LLMs with DCD to take advantage of the prior knowledge learned by LLMs while maintaining the performance and computational efficiency of DCD approaches for causal discovery. LLM-DCD outperforms previous state-of-the-art approaches across several causal discovery benchmarking datasets. Although not as efficient as the recent developed SDCD, LLM-DCD is more scalable than several other baseline SBM and DCD methods. Future implementations of LLM-DCD may be able to directly integrate the computational optimization of SDCD \cite{nazaret2024stabledifferentiablecausaldiscovery} to improve scalability while preserving state-of-the-art performance.

Because our approach also directly optimizes an explicitly defined adjacency matrix, LLM-DCD also provides a more interpretable approach to causal discovery. Additionally, LLM-DCD directly benefits from higher quality initializations of this adjacency matrix, and can take advantage of future advancements in LLM reasoning and causal inference capabilities. Other future work may investigate how the size and reasoning capabilities of various LLMs can affect initialization of the adjacency matrix and downstream performance of LLM-DCD. As LLM-based causal discovery evolves, the performance of LLM-DCD is only expected to improve as well.

\bibliography{final_paper}

\begin{thebibliography}{17}
\providecommand{\natexlab}[1]{#1}
\providecommand{\url}[1]{\texttt{#1}}
\expandafter\ifx\csname urlstyle\endcsname\relax
  \providecommand{\doi}[1]{doi: #1}\else
  \providecommand{\doi}{doi: \begingroup \urlstyle{rm}\Url}\fi

\bibitem[Ban et~al.(2023)Ban, Chen, Wang, and Chen]{ban2023querytoolscausalarchitects}
T.~Ban, L.~Chen, X.~Wang, and H.~Chen.
\newblock From query tools to causal architects: Harnessing large language models for advanced causal discovery from data, 2023.
\newblock URL \url{https://arxiv.org/abs/2306.16902}.

\bibitem[Bello et~al.(2023)Bello, Aragam, and Ravikumar]{bello2023dagmalearningdagsmmatrices}
K.~Bello, B.~Aragam, and P.~Ravikumar.
\newblock Dagma: Learning dags via m-matrices and a log-determinant acyclicity characterization, 2023.
\newblock URL \url{https://arxiv.org/abs/2209.08037}.

\bibitem[Brouillard et~al.(2020)Brouillard, Lachapelle, Lacoste, Lacoste-Julien, and Drouin]{brouillard2020differentiablecausaldiscoveryinterventional}
P.~Brouillard, S.~Lachapelle, A.~Lacoste, S.~Lacoste-Julien, and A.~Drouin.
\newblock Differentiable causal discovery from interventional data, 2020.
\newblock URL \url{https://arxiv.org/abs/2007.01754}.

\bibitem[Chickering et~al.(2004)Chickering, Heckerman, and Meek]{chickering}
D.~M. Chickering, D.~Heckerman, and C.~Meek.
\newblock Large-sample learning of bayesian networks is np-hard.
\newblock \emph{J. Mach. Learn. Res.}, 5:\penalty0 1287–1330, dec 2004.
\newblock ISSN 1532-4435.

\bibitem[Darvariu et~al.(2024)Darvariu, Hailes, and Musolesi]{darvariu2024largelanguagemodelseffective}
V.-A. Darvariu, S.~Hailes, and M.~Musolesi.
\newblock Large language models are effective priors for causal graph discovery, 2024.
\newblock URL \url{https://arxiv.org/abs/2405.13551}.

\bibitem[Jin et~al.(2024)Jin, Chen, Leeb, Gresele, Kamal, Lyu, Blin, Adauto, Kleiman-Weiner, Sachan, and Schölkopf]{jin2024cladderassessingcausalreasoning}
Z.~Jin, Y.~Chen, F.~Leeb, L.~Gresele, O.~Kamal, Z.~Lyu, K.~Blin, F.~G. Adauto, M.~Kleiman-Weiner, M.~Sachan, and B.~Schölkopf.
\newblock Cladder: Assessing causal reasoning in language models, 2024.
\newblock URL \url{https://arxiv.org/abs/2312.04350}.

\bibitem[Jiralerspong et~al.(2024)Jiralerspong, Chen, More, Shah, and Bengio]{jiralerspong2024efficientcausalgraphdiscovery}
T.~Jiralerspong, X.~Chen, Y.~More, V.~Shah, and Y.~Bengio.
\newblock Efficient causal graph discovery using large language models, 2024.
\newblock URL \url{https://arxiv.org/abs/2402.01207}.

\bibitem[Kingma and Ba(2015)]{adam}
D.~P. Kingma and J.~Ba.
\newblock Adam: {A} method for stochastic optimization.
\newblock In Y.~Bengio and Y.~LeCun, editors, \emph{3rd International Conference on Learning Representations, {ICLR} 2015, Conference Track Proceedings}, 2015.
\newblock URL \url{http://arxiv.org/abs/1412.6980}.

\bibitem[Kıcıman et~al.(2023)Kıcıman, Ness, Sharma, and Tan]{kıcıman2023causalreasoninglargelanguage}
E.~Kıcıman, R.~Ness, A.~Sharma, and C.~Tan.
\newblock Causal reasoning and large language models: Opening a new frontier for causality, 2023.
\newblock URL \url{https://arxiv.org/abs/2305.00050}.

\bibitem[Liu et~al.(2024)Liu, Xu, Wu, Yuan, Yang, Zhou, Liu, Guan, Wang, Yu, McAuley, Ai, and Huang]{liu2024largelanguagemodelscausal}
X.~Liu, P.~Xu, J.~Wu, J.~Yuan, Y.~Yang, Y.~Zhou, F.~Liu, T.~Guan, H.~Wang, T.~Yu, J.~McAuley, W.~Ai, and F.~Huang.
\newblock Large language models and causal inference in collaboration: A comprehensive survey, 2024.
\newblock URL \url{https://arxiv.org/abs/2403.09606}.

\bibitem[Long et~al.(2023)Long, Piché, Zantedeschi, Schuster, and Drouin]{long2023causaldiscoverylanguagemodels}
S.~Long, A.~Piché, V.~Zantedeschi, T.~Schuster, and A.~Drouin.
\newblock Causal discovery with language models as imperfect experts, 2023.
\newblock URL \url{https://arxiv.org/abs/2307.02390}.

\bibitem[Meek(1997)]{Meek1997}
C.~Meek.
\newblock {Graphical Models: Selecting causal and statistical models}.
\newblock 1 1997.
\newblock \doi{10.1184/R1/22696393.v1}.
\newblock URL \url{https://kilthub.cmu.edu/articles/thesis/Graphical_Models_Selecting_causal_and_statistical_models/22696393}.

\bibitem[Nazaret et~al.(2024)Nazaret, Hong, Azizi, and Blei]{nazaret2024stabledifferentiablecausaldiscovery}
A.~Nazaret, J.~Hong, E.~Azizi, and D.~Blei.
\newblock Stable differentiable causal discovery, 2024.
\newblock URL \url{https://arxiv.org/abs/2311.10263}.

\bibitem[Scutari(2010)]{bnlearnpaper}
M.~Scutari.
\newblock Learning bayesian networks with the bnlearn r package.
\newblock \emph{Journal of Statistical Software}, 35\penalty0 (3):\penalty0 1–22, 2010.
\newblock \doi{10.18637/jss.v035.i03}.
\newblock URL \url{https://www.jstatsoft.org/index.php/jss/article/view/v035i03}.

\bibitem[Vashishtha et~al.(2023)Vashishtha, Reddy, Kumar, Bachu, Balasubramanian, and Sharma]{vashishtha2023causalinferenceusingllmguided}
A.~Vashishtha, A.~G. Reddy, A.~Kumar, S.~Bachu, V.~N. Balasubramanian, and A.~Sharma.
\newblock Causal inference using llm-guided discovery, 2023.
\newblock URL \url{https://arxiv.org/abs/2310.15117}.

\bibitem[Waxman et~al.(2024)Waxman, Butler, and Djurić]{Waxman_2024}
D.~Waxman, K.~Butler, and P.~M. Djurić.
\newblock Dagma-dce: Interpretable, non-parametric differentiable causal discovery.
\newblock \emph{IEEE Open Journal of Signal Processing}, 5:\penalty0 393–401, 2024.
\newblock ISSN 2644-1322.
\newblock \doi{10.1109/ojsp.2024.3351593}.
\newblock URL \url{http://dx.doi.org/10.1109/OJSP.2024.3351593}.

\bibitem[Zheng et~al.(2018)Zheng, Aragam, Ravikumar, and Xing]{zheng2018dagstearscontinuousoptimization}
X.~Zheng, B.~Aragam, P.~Ravikumar, and E.~P. Xing.
\newblock Dags with no tears: Continuous optimization for structure learning, 2018.
\newblock URL \url{https://arxiv.org/abs/1803.01422}.

\end{thebibliography}


\appendix

\section{Appendix}

\subsection{Implementation details of \texttt{MLE-INTERP} and gradient ascent optimization}
Algorithm (\ref{alg:mleinterp}) details our implementation of Eq.(\ref{eq:mleinterp}) in the main text. 
\begin{algorithm}[h]
\caption{\quad $\texttt{MLE-INTERP}(x^i_j \, ; \, X,A)$}\label{alg:mleinterp}
\begin{algorithmic}
    \State{\textbf{Input}: $\; i, j, X, A$}
    \State{$\texttt{num, den} \gets 0$}
    \For{$k \in [n]$}
        \State{$\texttt{numprod, denprod} \gets 1$}
        \For{$m \in [d]$}
            \State{$\texttt{numprod} \gets \texttt{numprod} \cdot \left( ( 1 \text{ if } \mathbf{x}^k_m = \mathbf{x}^i_m \text{ else } g(1 -a_{mj}) ) \text{ if } \mathbf{x}^k_j = \mathbf{x}^i_j \text{ else } 0 \right)$}
            \State{$\texttt{denprod} \gets \texttt{denprod} \cdot ( 1 \text{ if } \mathbf{x}^k_m = \mathbf{x}^i_m \text{ or } m = j \text{ else } g(1 -a_{mj}) )$}
        \EndFor
        \State{$\texttt{num} \gets \texttt{num} + \texttt{numprod}$}
        \State{$\texttt{den} \gets \texttt{den} + \texttt{denprod}$}
    \EndFor
    \State \Return \, (\texttt{num / den})
\end{algorithmic}
\end{algorithm}

The maximization of the objective function is carried out using the \texttt{LLM-DCD} algorithm summarized in Algorithm (\ref{alg:llm_dcd}). For more details on hyperparameter settings we refer the reader to Appendix \ref{app:hyperp}.

\begin{algorithm}[h]
\caption{\quad $\texttt{LLM-DCD}$}\label{alg:llm_dcd}
\begin{algorithmic}
    \State{\textbf{Input}: $\; X, A_0, i_1, i_2, \mathrm{lr}, \alpha, \delta, \gamma_1, \gamma_2, \delta, \varepsilon$}
    \State{$A_t \gets A_0$}
    \State{$M_t, V_t \gets 0^{d \times d}$}
    \State{$t, \beta_t \gets 0$}
    \While{$t \le i_1 + i_2$ and \textit{not converged}}
        \State{$t \gets t + 1$}
        \State{$\texttt{D} \gets \text{gradient of objective function $\mathcal{L}(X ;A)$ with respect to } W$}
        \State{$M_t \gets ({\gamma_1} M_t + (1 - \gamma_1) \texttt{D}) / \gamma_1^t$}
        \State{$V_t \gets ({\gamma_2} V_t + (1 - \gamma_2) (\texttt{D} \odot \texttt{D}) ) / \gamma_2^t$}
        \State{$A_t \gets \text{clip}(A_t + \mathrm{lr} \cdot \texttt{divide}(M_t, \sqrt{V_t} + \varepsilon), 0, 1)$}
        \State{$\beta_t \gets \beta_t + (\delta \text{ if } t > i_1 \text{ else } 0)$}
    \EndWhile
\end{algorithmic}
\end{algorithm}
In \texttt{LLM-DCD}, $\odot$ denotes element-wise multiplication and $\texttt{divide}( \cdot )$ denotes element-wise division. The gradient, \texttt{D}, of the objective function can be computed using $\texttt{D-MLE-INTERP}(x^i_j \, ; \, C, A)$, summed over all $i, j$ and by computing the derivative of the spectral acyclicity constraint, $h(A)$.
For an algorithm to calculate \texttt{D-MLE-INTERP} with respect to $W$, see Appendix (\ref{app:alg}).

\subsection{LLM-DCD Hyperparameters} \label{app:hyperp}

For the LLM-DCD algorithm specified in Algorithm \ref{alg:llm_dcd}, we use the following hyperparameters: $i_1 = 200, i_2 = 400, \mathrm{lr} = 0.025, \alpha = 0.120, \delta=0.040, \gamma_1 = 0.80, \gamma_2 = 0.90, \varepsilon = 10^{-8}$. The only change across all datasets was in the minibatch-size ($b$) selected ($b = 125$ for \texttt{hepar2}, $b = 250$ for \texttt{alarm} and $b = 500$ for other datasets). Black-box implementations for all other methods were used with no modifications, except for DAGMA, for which we ran $T=1$ iterations instead of $T=4$, to keep runtimes comparable with other methods. For LLM methods, we used OpenAI GPT-4 (\texttt{gpt-4-0613}) with default settings. All experiments were conducted on a virtual machine provided by Google Colaboratory, utilizing an NVIDIA A100 Tensor Core GPU.

\subsection{Time-complexity analysis} \label{app:time}
$\forall i, j, \; \texttt{D-MLE-INTERP}(x^i_j ;  X, A)$ (Appendix \ref{app:alg}) can be computed in $O(n d^3)$-time. Since $\texttt{D}$ requires a sum over all $i, j$, $O(n^2 d^4)$-time is needed. This computation, however, is easily parallelizable; using a GPU, only $O(nd)$ steps need to be carried out sequentially. Using mini-batch gradient descent, we can reduce the time-complexity to $O(b^2 d^4)$-time (with only $O(bd)$ sequential steps) ($b =$ mini-batch size). Computing $h'(A)$ naïvely requires $O(d^3)$-time in our current implementation, although we plan to use the $O(d^2)$-time power-iteration algorithm from\cite{nazaret2024stabledifferentiablecausaldiscovery} in future studies.

\subsection{Derivative of \texttt{MLE-INTERP}} \label{app:alg}
The derivative of \texttt{MLE-INTERP} (\texttt{D-MLE-INTERP}) with respect to $W$.

\begin{algorithm}[h]
\caption{\quad $\texttt{D-MLE-INTERP}(x^i_j \, ; \, X, A)$}
\begin{algorithmic}
    \State{\textbf{Input}: $\; i, j, X, A$}
    \State{$\texttt{DLL} \gets 0^{d \times d}$}
    \For{$a \in [d]$}
        \For{$b \in [d] : a \ne b$}
            \State{$\texttt{num, den, dnum, dden} \gets 0$}
            \For{$k \in [n]$}
            \State{$\texttt{numprod, denprod, dnumprod, ddenprod} \gets 1$}
                \For{$m \in [d]$}  
                    \State{$\texttt{numterm} \gets \left( ( 1 \text{ if } \mathbf{x}^k_m = \mathbf{x}^i_m \text{ else } g(1 -a_{mj}) ) \text{ if } \mathbf{x}^k_j = \mathbf{x}^i_j \text{ else } 0 \right)$}
                    \State{$\texttt{denterm} \gets ( 1 \text{ if } \mathbf{x}^k_m = \mathbf{x}^i_m \text{ or } m = j \text{ else } g(1 - a_{mj}) )$}
                    \State{$\texttt{numprod} \gets \texttt{numprod} \cdot \texttt{numterm}$}
                    \State{$\texttt{denprod} \gets \texttt{denprod} \cdot \texttt{denterm}$}
                    \If{$a = m$ and $b = j$}
                        \State{$\texttt{dnumprod} \gets \texttt{dnumprod} \cdot ( 0 \text{ if } \mathbf{x}^k_m = \mathbf{x}^i_m \text{ or } \mathbf{x}^k_j = \mathbf{x}^i_j \text{ else } g'(1 - a_{mj}) )$}
                        \State{$\texttt{ddenprod} \gets \texttt{ddenprod} \cdot ( 0 \text{ if } \mathbf{x}^k_m = \mathbf{x}^i_m \text{ or } m = j \text{ else } g'(1 - a_{mj}) )$}
                    \Else
                        \State{$\texttt{dnumprod} \gets \texttt{dnumprod} \cdot \texttt{numterm} $}
                        \State{$\texttt{ddenprod} \gets \texttt{ddenprod} \cdot \texttt{denterm} $}
                    \EndIf
                \EndFor
                \State{$\texttt{num} \gets \texttt{num} + \texttt{numprod}$}
                \State{$\texttt{den} \gets \texttt{den} + \texttt{denprod}$}
                \State{$\texttt{dnum} \gets \texttt{dnum} + \texttt{dnumprod}$}
                \State{$\texttt{dden} \gets \texttt{dden} + \texttt{ddenprod}$}
            \EndFor
            \State $\texttt{DLL}[a, b] \gets (\texttt{dnum / num} - \texttt{dden / den})$
        \EndFor
    \EndFor
    \State \Return \, $\texttt{DLL}$
\end{algorithmic}
\end{algorithm}

\subsection{Regularity assumptions} \label{app:proof}

\begin{theorem}[\cite{brouillard2020differentiablecausaldiscoveryinterventional}]
    Under the following \textbf{regularity assumptions} (these assumptions have been simplified from the original assumptions listed in \cite{brouillard2020differentiablecausaldiscoveryinterventional}, since we do not consider interventions):
    \begin{enumerate}[noitemsep]
        \item $n \ge 0$
        \item Faithfulness; causal Markov condition; causal sufficiency.
        \item The joint-distribution of the latent or ground-truth CGM and the all distributions from the model class ($\texttt{MLE-INTERP}$) have strictly-positive density.
        \item The model class is able to express the latent CGM conditional distributions in the limit of infinite samples ($n \to \infty$). This assumption is implicit in their work.
    \end{enumerate}
    Any CGM that maximizes of the DCD objective function is equivalent to the ground-truth CGM upto a Markov equivalence class (i.e., the CGM is Markov equivalent to the ground-truth CGM).
\end{theorem}

\begin{theorem}[Regularity of \texttt{LLM-DCD}]
    The model used by \texttt{LLM-DCD} satisfies regularity assumptions, in the limit of infinite observational samples ($n \to \infty$).
\end{theorem}

\begin{proof}
    Assumption 1 is trivially satisfied, since we always include a non-empty observational data table $X$. Assumption 2 is included in the assumptions of this work. For Assumption 3, refer to the implementation of \texttt{MLE-INTERP}. The numerator term (\texttt{num}) is non-zero whenever there is at least one row in $X$ where every variable of interest takes on every one of its possible values from its finite, discrete set of values. If the joint-distribution of the latent or ground-truth CGM has strictly positive density (as is the case for all chosen datasets), Assumption 3 holds true for all distributions from the model class $\texttt{MLE-INTERP}$ in the limit of infinite samples ($n \to \infty$). Finally, we must show that Assumption 4 holds true for $\texttt{MLE-INTERP}$.

    Let $W^*$ be the adjacency matrix of the latent CGM DAG, and $W$ be the adjacency matrix used in $\texttt{LLM-DCD}$. We will show that when $W = W^*$, $\texttt{MLE-INTERP}$ models the latent CGM joint distribution. Based on the implementation of $\texttt{MLE-INTERP}$:
    \begin{align*}
        \texttt{MLE-INTERP}(x^i_j  ;  X, A^*) &= \frac{\text{cnt}\left(\mathbf{x}^i_j, \mathrm{Pa}(v_j, A^*) = \mathbf{x}^i_{\mathrm{Pa}(v_j, A^*)}\right)}{\text{cnt}\left(\mathrm{Pa}(v_j, A^*) = \mathbf{x}^i_{\mathrm{Pa}(v_j, A^*)}\right)} \\
        &\to \Pr[ v_j = \mathbf{x}^i_j \; | \; \mathrm{Pa}(v_j, A^*) = \mathbf{x}^i_{\mathrm{Pa}(v_j, A^*)}] \quad (\text{as } n \to \infty)
    \end{align*}
    Recall that $\mathrm{Pa}(v_j, A^*)$ is defined as the causal parents or direct causes of $v_j$ in the latent CGM (according to $A^*$). $\text{count}_X$ represents the number of observations or rows in the table $X$ that satisfy the specified condition. Thus, we have shown that, our specified model \texttt{MLE-INTERP} is able to model the conditional distributions in the latent CGM (as $n \to \infty$).
\end{proof}

\subsection{Other relevant metrics} \label{app:figure}

The following tables shows how relevant metrics scale with $d$ for all methods. The results follow the same trends as for SHD in Figure \ref{SHD_table}. Comparable results within the margin of error of the best performing algorithm are bolded. Note that the Recall values for the PAIR method are particularly high: this is because the pairwise LLM method tends to prioritize false positives over false negatives based on the provided correlation coefficients. This trade-off may be adjusted in future work.


\begin{table}[!h]
\centering
\caption{Performance comparison of methods on the cancer dataset ($d = 5$)}
\label{cancer-performance-table}
\begin{tabular}{lcccc}
\toprule
\textbf{Method} & \textbf{Precision} & \textbf{Recall} & \textbf{F1-score} & \textbf{Time (s)} \\
\midrule
SDCD	       & $0.56 \pm 0.00$          & $0.50 \pm 0.00$          & $0.52 \pm 0.04$          & $88.1 \pm 11.4$ \\
GES	           & $\mathbf{0.67 \pm 0.29}$ & $0.33 \pm 0.29$          & $0.44 \pm 0.38$          & $\mathbf{0.09 \pm 0.00}$ \\
DAGMA	       & $0.00 \pm 0.00$          & $0.00 \pm 0.00$          & $0.00 \pm 0.00$          & $215. \pm 1.82$ \\
PAIR	       & $0.50 \pm 0.00$          & $0.50 \pm 0.00$          & $0.50 \pm 0.00$          & $8.52 \pm 1.31$ \\
BFS	           & $\mathbf{0.82 \pm 0.00}$ & $\mathbf{1.00 \pm 0.00}$ & $\mathbf{0.90 \pm 0.10}$ & $67.2 \pm 8.81$ \\
LLM-DCD	       & $0.33 \pm 0.14$          & $0.08 \pm 0.14$          & $0.13 \pm 0.23$          & $2.14 \pm 0.19$ \\
LLM-DCD-P & $\mathbf{0.67 \pm 0.14}$          & $0.33 \pm 0.14$          & $0.44 \pm 0.19$          & $1.38 \pm 0.21$ \\
LLM-DCD-B  & $\mathbf{0.82 \pm 0.00}$ & $\mathbf{1.00 \pm 0.00}$ & $\mathbf{0.90 \pm 0.10}$ & $0.50 \pm 0.07$ \\
\bottomrule
\end{tabular}
\end{table}

\begin{table}[!h]
\centering
\caption{Performance comparison of methods on the sachs dataset ($d = 11$)}
\label{sachs-performance-table}
\begin{tabular}{lcccc}
\toprule
\textbf{Method} & \textbf{Precision} & \textbf{Recall} & \textbf{F1-score} & \textbf{Time (s)} \\
\midrule
SDCD	       & $\mathbf{0.49 \pm	0.12}$ &	$0.39 \pm 0.12$          & $\mathbf{0.43 \pm 0.13}$ & $86.5 \pm 17.5$ \\
GES	           & $0.29 \pm	0.09$          &	$0.16 \pm 0.09$          & $0.20 \pm 0.11$          & $\mathbf{1.49 \pm 0.16}$ \\
DAGMA	       & $0.03 \pm	0.03$          &	$0.02 \pm 0.03$          & $0.02 \pm 0.04$          & $240. \pm 2.64$ \\
PAIR	       & $0.27 \pm	0.00$          &	$\mathbf{1.00 \pm 0.00}$ & $\mathbf{0.43 \pm 0.00}$ & $54.9 \pm 8.01$ \\
BFS	           & $\mathbf{0.53 \pm	0.07}$ &	$0.10 \pm 0.07$          & $0.17 \pm 0.12$          & $129. \pm 37.4$ \\
LLM-DCD	       & $0.28 \pm	0.16$          &	$0.35 \pm 0.16$          & $\mathbf{0.30 \pm 0.15}$ & $5.26 \pm 1.91$ \\
LLM-DCD-P & $0.31 \pm	0.03$          &	$0.39 \pm 0.03$          & $\mathbf{0.36 \pm 0.04}$ & $\mathbf{4.75 \pm 2.06}$ \\
LLM-DCD-B  & $\mathbf{0.49 \pm	0.03}$ &	$0.37 \pm 0.03$          & $\mathbf{0.42 \pm 0.03}$ & $\mathbf{4.22 \pm 2.05}$ \\
\bottomrule
\end{tabular}
\end{table}

\begin{table}[H]
\centering
\caption{Performance comparison of methods on the child dataset ($d = 20$)}
\label{child-performance-table}
\begin{tabular}{lcccc}
\toprule
\textbf{Method} & \textbf{Precision} & \textbf{Recall} & \textbf{F1-score} & \textbf{Time (s)} \\
\midrule
SDCD	       & $\mathbf{0.28 \pm 0.12}$ & $0.49 \pm 0.12$          & $\mathbf{0.36 \pm 0.07}$ & $61.3 \pm	7.00$ \\
GES	           & $0.20 \pm 0.02$          & $0.21 \pm 0.02$          & $0.22 \pm 0.03$          & $\mathbf{28.5 \pm	3.35}$ \\
DAGMA	       & $0.04 \pm 0.02$          & $0.01 \pm 0.02$          & $0.05 \pm 0.08$          & $277. \pm	5.23$ \\
PAIR	       & $0.13 \pm 0.02$          & $\mathbf{0.99 \pm 0.02}$ & $0.18 \pm 0.10$          & $222. \pm 19.0$ \\
BFS	           & $0.07 \pm 0.06$          & $0.05 \pm 0.06$          & $0.10 \pm 0.09$          & $150. \pm	38.0$ \\
LLM-DCD	       & $0.21 \pm 0.10$          & $0.41 \pm 0.10$          & $0.19 \pm 0.10$          & $50.1 \pm	0.13$ \\
LLM-DCD-P & $\mathbf{0.30 \pm 0.08}$ & $0.61 \pm 0.08$          & $\mathbf{0.41 \pm 0.04}$ & $49.5 \pm	0.11$ \\
LLM-DCD-B  & $\mathbf{0.41 \pm 0.00}$ & $0.32 \pm 0.00$          & $\mathbf{0.39 \pm 0.06}$ & $48.8 \pm	0.12$ \\
\bottomrule
\end{tabular}
\end{table}

\begin{table}[H]
\centering
\caption{Performance comparison of methods on the alarm dataset ($d = 37$)}
\label{alarm-performance-table}
\begin{tabular}{lcccc}
\toprule
\textbf{Method} & \textbf{Precision} & \textbf{Recall} & \textbf{F1-score} & \textbf{Time (s)} \\
\midrule
SDCD	       & $0.08 \pm 0.05$          & $0.23 \pm 0.05$          & $0.12 \pm	0.02$          & $\mathbf{78.9 \pm	4.35}$ \\
GES	           & $0.06 \pm 0.02$          & $0.08 \pm 0.02$          & $0.07 \pm	0.01$          & $341. \pm	2.10$ \\
DAGMA	       & $0.06 \pm 0.02$          & $0.07 \pm 0.02$          & $0.06 \pm	0.02$          & $393. \pm	20.7$ \\
PAIR	       & $0.11 \pm 0.01$          & $\mathbf{0.97 \pm 0.01}$ & $0.20 \pm	0.01$          & $636. \pm 29.9$ \\
BFS	           & $0.26 \pm 0.08$          & $0.22 \pm 0.08$          & $0.24 \pm	0.08$          & $290. \pm	47.9$ \\
LLM-DCD	       & $0.30 \pm 0.02$          & $0.35 \pm 0.02$          & $0.32 \pm	0.02$          & $395. \pm	0.25$ \\
LLM-DCD-P & $0.33 \pm 0.00$          & $0.43 \pm 0.00$          & $0.37 \pm	0.06$          & $395. \pm	0.21$ \\
LLM-DCD-B  & $\mathbf{0.57 \pm 0.07}$ & $0.46 \pm 0.07$          & $\mathbf{0.51 \pm	0.07}$ & $394. \pm	0.20$ \\
\bottomrule
\end{tabular}
\end{table}

\begin{table}[H]
\centering
\caption{Performance comparison of methods on the hepar2 dataset ($d = 70$)}
\label{hepar2-performance-table}
\begin{tabular}{lcccc}
\toprule
\textbf{Method} & \textbf{Precision} & \textbf{Recall} & \textbf{F1-score} & \textbf{Time (s)} \\
\midrule
SDCD	        & $0.12 \pm 0.06$          & $\mathbf{0.18 \pm 0.06}$ & $\mathbf{0.14 \pm 0.04}$ & $\mathbf{66.5 \pm	26.2}$ \\
GES             & $-$                      & $-$                      & $-$                      & $-$ \\           	          
DAGMA	        & $0.02 \pm 0.01$          & $0.03 \pm 0.01$          & $0.03 \pm 0.01$          & $638. \pm	8.75$ \\
PAIR            & $-$                      & $-$                      & $-$                      & $-$ \\
BFS	            & $0.17 \pm 0.02$          & $\mathbf{0.11 \pm 0.02}$ & $\mathbf{0.13 \pm 0.03}$	& $544. \pm	50.4$ \\
LLM-DCD	        & $0.03 \pm 0.00$          & $0.05 \pm 0.00$          & $0.03 \pm 0.00$          & $623. \pm	0.26$ \\
LLM-DCD-P  & $-$                      & $-$                      & $-$                      & $-$ \\
LLM-DCD-B	& $\mathbf{0.28 \pm 0.03}$ & $\mathbf{0.13 \pm 0.03}$ & $\mathbf{0.18 \pm 0.03}$ & $623. \pm	0.38$ \\
\bottomrule
\end{tabular}
\end{table}



\end{document}